\def\year{2022}\relax
\documentclass[letterpaper]{article} 
\usepackage{aaai22}  
\usepackage{times}  
\usepackage{helvet}  
\usepackage{courier}  
\usepackage[hyphens]{url}  
\usepackage{graphicx} 
\usepackage{natbib}  
\usepackage{caption} 
\DeclareCaptionStyle{ruled}{labelfont=normalfont,labelsep=colon,strut=off} 
\usepackage{amsmath,amssymb,amsthm}
\usepackage{algorithm}
\usepackage[noend]{algpseudocode}

\usepackage{newfloat}
\usepackage{listings}
\floatstyle{ruled}
\newfloat{listing}{tb}{lst}{}
\floatname{listing}{Listing}
\title{Pretrained Cost Model for Distributed Constraint Optimization Problems}
\author {
    Yanchen Deng,
    Shufeng Kong\footnote{Corresponding Author},
    Bo An
}
\affiliations {
    School of Computer Science and Engineering, Nanyang Technological University, Singapore\\
    \{ycdeng, shufeng.kong, boan\}@ntu.edu.sg
}

\begin{document}

\maketitle

\begin{abstract}
Distributed Constraint Optimization Problems (DCOPs) are an important subclass of combinatorial optimization problems, where information and controls are distributed among multiple autonomous agents. 
Previously, Machine Learning (ML) has been largely applied to solve combinatorial optimization problems by learning effective heuristics. However, existing ML-based heuristic methods are often not generalizable to different search algorithms.
Most importantly, these methods usually require full knowledge about the problems to be solved, which are not suitable for distributed settings where centralization is not realistic due to geographical limitations or privacy concerns. To address the generality issue, we propose a novel directed acyclic graph representation schema for DCOPs and leverage the Graph Attention Networks (GATs) to embed graph representations. Our model, GAT-PCM, is then pretrained with optimally labelled data in an offline manner, so as to construct effective heuristics to boost a broad range of DCOP algorithms where evaluating the quality of a partial assignment is critical, such as local search or backtracking search. Furthermore, to enable decentralized model inference, we propose a distributed embedding schema of GAT-PCM where each agent exchanges only embedded vectors, and show its soundness and complexity.
Finally, we demonstrate the effectiveness of our model by combining it with a local search or a backtracking search algorithm. Extensive empirical evaluations indicate that the GAT-PCM-boosted algorithms significantly outperform the state-of-the-art methods in various benchmarks. Our pretrained cost model is available at https://github.com/dyc941126/GAT-PCM.

\end{abstract}

\section{Introduction}
As a fundamental formalism in multi-agent systems, Distributed Constraint Optimization Problems (DCOPs) \cite{modi2005adopt} capture the essentials of cooperative distributed problem solving and have been successfully applied to model the problems in many real-world domains like radio channel allocation \cite{monteiro2012multi}, vessel navigation \cite{hirayama2019dssa+}, and smart grid \cite{fioretto2017distributed}.

Over the past two decades, numerous algorithms have been proposed to solve DCOPs and can be generally classified as complete and incomplete algorithms. Complete algorithms aim to exhaust the search space and find the optimal solution by either distributed backtracking search \cite{hirayamay97,modi2005adopt,litov2017forward,yeoh2010bnb} or dynamic-programming \cite{chen2020rmb,petcuscalable,petcu2007mb}. However, complete algorithms scale poorly and are unsuitable for large real-world applications. Therefore, considerable research efforts have been devoted to develop incomplete algorithms that trade the solution quality for smaller computational overheads, including local search \cite{maheswaran2004distributed,okamoto2016distributed,Zhang2005Distributed}, belief propagation \cite{cohen2020governing,farinelli2008decentralised,rogers2011bounded,zivan2017balancing,chenDWH18} and sampling \cite{nguyen2019distributed,ottens2017duct}.

However, the existing DCOP algorithms usually rely on handcrafted heuristics which need expertise to tune for different settings. In contrast, Machine Learning (ML) based techniques learn effective heuristics for existing methods automatically \cite{bengio2021machine,gasseCFC019,Lederman2020Learning}, achieving state-of-the-art performance in various challenging problems like Mixed Integer Programming (MIP), Capacitated Vehicle Routing Problems (CVRPs), and Boolean Satisfiability Problems (SATs). Unfortunately, these methods are often not generalizable to different search algorithms. Most importantly, many of these methods usually require the full knowledge about the problems to be solved, making them unsuitable for a distributed setting where centralization is not realistic due to geographical limitations or privacy concerns.

Therefore, we develop the first general-purpose ML model, named GAT-PCM, to generate effective heuristics for a wide range of DCOP algorithms and propose a distributed embedding schema of GAT-PCM for decentralized model inference. Specifically, we make the following key contributions: (1) We propose a novel directed tripartite graph representation based on microstructure \cite{jegou1993decomposition} to encode a partially instantiated DCOP instance and use Graph Attention Networks (GATs) \cite{vaswani2017attention} to learn generalizable embeddings. (2) Instead of generating heuristics for a particular algorithm, GAT-PCM predicts the optimal cost of a target assignment given a partial assignment, such that it can be applied to boost the performance of a wide range of DCOP algorithms where evaluating the quality of an assignment is critical.
To this end, we pretrain our model on a dataset where DCOP instances are sampled from a problem distribution, partial assignments are constructed according to pseudo trees, and cost labels are generated by a complete algorithm. 
(3) We propose a Distributed Embedding Schema (DES) to perform decentralized model inference without disclosing local constraints, where each agent exchanges only the embedded vectors via localized communication. We also theoretically show the correctness and complexity of DES.
(4) As a case study, we develop two efficient heuristics for DLNS \cite{hoang2018large} and backtracking search for DCOPs based on GAT-PCM, respectively. Specifically, by greedily constructing a solfution, our GAT-PCM can serve as a subroutine of DLNS to repair assignments. Besides, the predicted cost of each assignment is used as a criterion for domain ordering in backtracking search. (5) Extensive empirical evaluations indicate that GAT-PCM-boosted algorithms significantly outperform the state-of-the-art methods in various standard benchmarks.

\section{Related Work}
There is an increasing interest of applying neural networks to solve SAT problems in recent years. Selsam et al. (\citeyear{selsamLBLMD19}) proposed NeuroSAT, a message passing neural network built upon LSTMs \cite{hochreiter1997long} to predict the satisfiability of a SAT and further decode the satisfying assignments. Yolcu and P{\'o}czos (\citeyear{yolcu2019learning}) proposed to use Graph Neural Networks (GNNs) to encode a SAT and REINFORCE \cite{williams1992simple} to learn local search heuristics. Similarly, Kurin et al. (\citeyear{kurin2020can}) proposed to learn branching heuristics for a CDCL solver \cite{een2003extensible} using GNNs and DQN \cite{mnih2015human}. Beside boolean formulas, Xu et al. (\citeyear{xu2018towards}) proposed to use CNNs \cite{cunBDHHHJ89} to predict the satisfiability of a general Constraint Satisfaction Problem (CSP). However, all of these methods require the total knowledge of a problem, making them unsuitable for distributed settings. Differently, our method uses an efficient distributed embedding schema to cooperatively compute the embeddings without disclosing constraints.

Very recently, there are two concurrent work \cite{yanchen21,razeghiKLBAD21} which uses Multi-layer Perceptrons (MLPs) to parameterize the high-dimensional data in traditional constraint reasoning techniques, e.g., Bucket Elimination \cite{dechter98}. Unfortunately, they follow an online learning strategy, which removes the most attractive feature of generalizing to new instances offered by neural networks. As a result, they require a significantly long runtime in order to train the MLPs. In contrast, we aim to develop an ML model for DCOPs which is supervisely pretrained with large-scale datasets beforehand. When applying the model to an instance, we just need several steps of model inference, which substantially reduces the overall overheads.

\section{Backgrounds}
In this section, we present preliminaries including DCOPs, GATs and pretrained models.
\paragraph{Distributed Constraint Optimization Problems} A Distributed Constraint Optimization Problem (DCOP) \cite{modi2005adopt} can be defined by a tuple $\langle I,X,D,F\rangle$ where $I=\{1,\dots,|I|\}$ is the set of agents, $X=\{x_1,\dots,x_{|X|}\}$ is the set of variables, $D=\{D_1,\dots,D_{|X|}\}$ is the set of discrete domains and $F=\{f_1,\dots,f_{|F|}\}$ is the set of constraint functions. Each variable $x_i$ takes a value from its domain $D_i$ and each function $f_i:D_{i_1}\times\cdots \times D_{i_k}\rightarrow \mathbb{R}_{\ge 0}$ defines the cost for each possible combination of $D_{i_1},\dots,D_{i_k}$.
Finally, the objective is to find a joint assignment $X\in D_1\times\dots\times D_{|X|}$ such that the following total cost is minimized:
\begin{equation}
\mathop{\min}\nolimits_X\sum\nolimits_{f_i\in F}f_i(X).
\end{equation}
For the sake of simplicity, we follow the common assumptions that each agent only controls a variable (i.e., $|I|=|X|$) and all constraints are binary (i.e., $f_{ij}: D_i\times D_j\rightarrow\mathbb{R}_{\ge 0}, \forall f_{ij}\in F$). Therefore, the term ``agent'' and ``variable'' can be used interchangeably and a DCOP can be visualized by a constraint graph in which vertices and edges represent the variables and constraints of the DCOP, respectively.

\paragraph{Graph Attention Networks}
Graph attention networks (GATs) \cite{velivckovic2017graph} are constructed by stacking a number of graph attention layers in which nodes are able to attend over their neighborhoods' features via the self-attention mechanism. 
Specifically, the attention coefficient between every pair of neighbor nodes is computed as $e_{ij}=a(\mathbf{W}h_i,\mathbf{W}h_j),$ where $h_i,h_j\in \mathbb{R}^d$ are node features, $\mathbf{W}\in \mathbb{R}^{d\times d}$ is a weight matrix, and $a$ is single-layer feed-forward neural network. Then the attention weight $\alpha_{ij}$ for nodes $j\in \mathcal{N}_i$ is computed as $\alpha_{ij}=\frac{\text{exp}(e_{ij})}{\sum_{k\in \mathcal{N}_i}\text{exp}(e_{ik})},$
where $\mathcal{N}_i$ is the neighborhood of node $v_i$ in the graph (including $v_i$).
At last, node $v_i$'s feature $h_i'$ is updated as $h_i'=g(\sum_{j\in\mathcal{N}_i}\alpha_{ij}\mathbf{W}h_j)$,
where $g$ is some nonlinear function such as the sigmoid. Multi-head attention \cite{vaswani2017attention} is also used where $K$ independent attention mechanisms are executed and their feature vectors are averaged as $h_i'=g(\frac{1}{K}\sum_{k=1}^{K}\sum_{j\in\mathcal{N}_i}\alpha_{ij}^k\mathbf{W}^k h_j).$
\begin{figure*}
    \centering
    \includegraphics[width=0.95\linewidth]{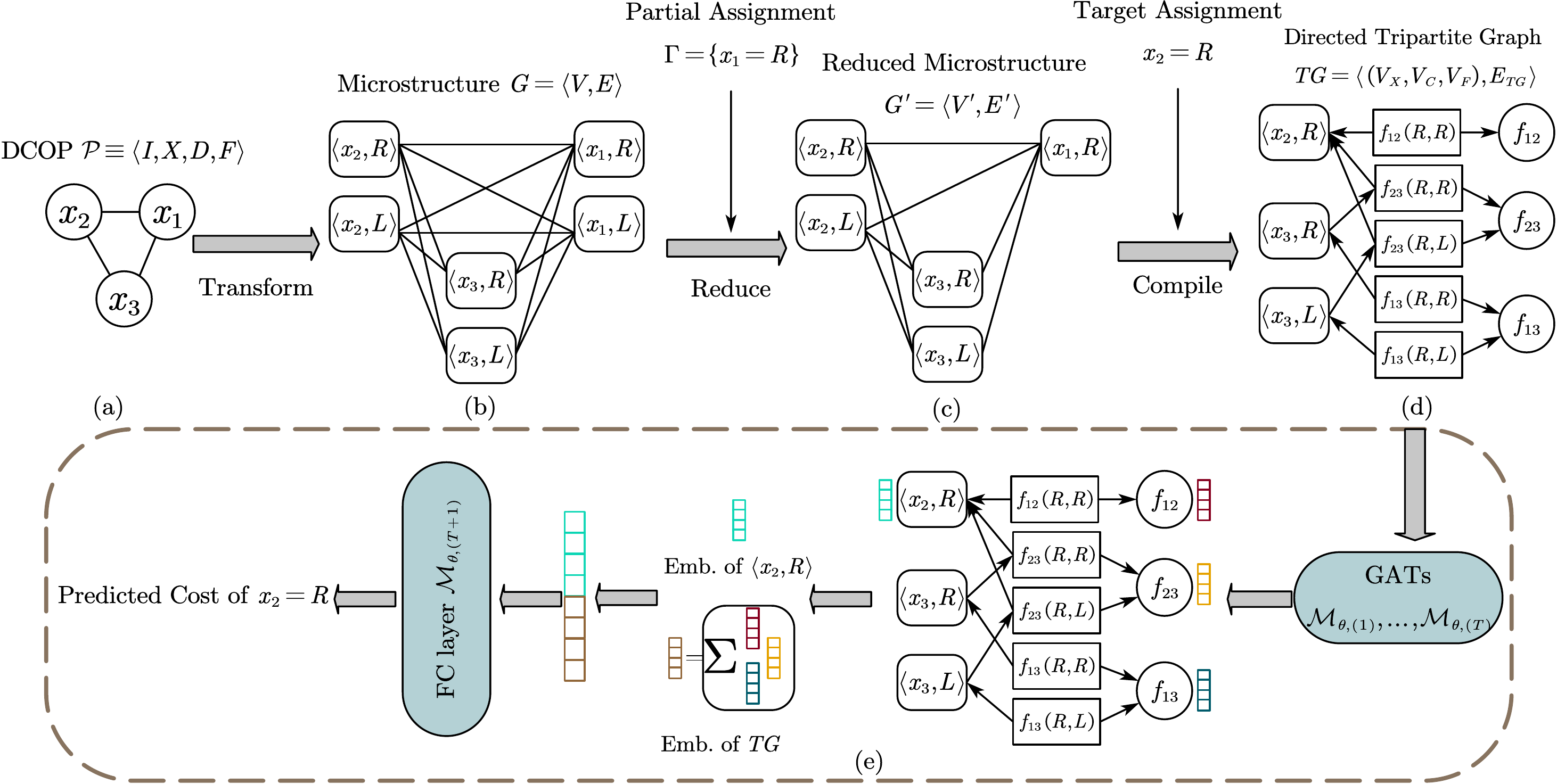}
    \caption{An illustration of the architecture of GAT-PCM with a small DCOP instance. A DCOP instance in (a) is first transformed into an equivalent microstructure $G$ in (b), and then $G$ is instantiated with a partial assignment $\Gamma$ by removing some nodes and edges in (c) and then further compiled to a directed tripartite graph with a given target assignment in (d) (cf. the section ``Graph Representations'' for details) in (e). Finally, we use GATs to learn an embedding with supervised training (cf. the sections ``Graph Embeddings" and ``Pretraining" for details).}
    \label{fig:arch}
\end{figure*}
\paragraph{Pretrained Models} The idea behind pretrained models is to first pretrain the models using large-scale datasets beforehand, then apply the models in downstream tasks to achieve state-of-the-art results. Beside significantly reducing the training overhead, pretrained models also offer substantial performance improvement over learning from scratch, leading to great successes in natural language processing \cite{brown2020language,devlin2018bert} and computer vision \cite{he2016deep,krizhevsky2017imagenet,simonyan2014very}.

In this work, we aim to develop the first effective and general-purpose pretrained model for DCOPs. In particular, we are interested in training a cost model $\mathcal{M}_\theta$ to predict the optimal cost of a partially instantiated DCOP instance, which is a core task in many DCOP algorithms:
\begin{equation}
\mathcal{M}_\theta(P,x_i=d_i;\Gamma)\mapsto \mathbb{R}, \label{eq-pcm-obj}  
\end{equation}
where $P\equiv\langle I,X,D,F\rangle$ is a DCOP instance, $\Gamma$ is a partial assignment, $x_i=d_i$ is the target assignment, and variable $x_i$ does not appear in $\Gamma$. This way, our cost model can be applied to a wide range of DCOP algorithms where evaluating the quality of an assignment is critical.

\section{Pretrained Cost Model for DCOPs}
In this section, we elaborate our pretrained cost model GAT-PCM. We begin with illustrating the architecture of the model in Fig.~\ref{fig:arch}. We then outline the centralized pretraining procedure for the model in Algo.~\ref{algo:train_c} to learn generalizable cost heuristics. We further propose a distributed embedding schema for decentralized model inference in Algo.~\ref{algo:train_d}.
Finally, we show how to use GAT-PCM to construct effective heuristics to boost DCOP algorithms.

The architecture of our GAT-PCM is illustrated in Fig.~\ref{fig:arch}. 
Recall that we aim to train a model to predict the optimal cost of a partially instantiated DCOP instance (cf. Eq.~(\ref{eq-pcm-obj})) and thus, we first need to embed a partially instantiated DCOP instance and the key is to build a suitable graph representation for a partially instantiated DCOP instance.
\subsubsection{Graph Representations}
Since DCOPs can be naturally represented by graphs with arbitrary sizes, we resort to GATs to learn generalizable and permutation-invariant representation of DCOPs. 
To this end, we first transform a DCOP instance $P\equiv\langle I,X,D,F\rangle$ to a microstructure representation \cite{jegou1993decomposition} where each variable assignment corresponds to a vertex and 
the constraint cost between a pair of vertices is represented by a weighted edge (cf. Fig.~\ref{fig:arch}(b)). After that, for each assignment $x_i= d_i$ in the partial assignment $\Gamma$, we remove all the other variable-assignment vertices of $x_i$, except $\langle x_i, d_i\rangle$, and their related edges from the microstructure (cf. Fig.~\ref{fig:arch}(c)). Then the reduced microstructure represents the partially instantiated DCOP instance w.r.t. $\Gamma$.

The reduced microstructure is further compiled into a directed tripartite graph $TG=\langle (V_X, V_C,V_F),E_{TG}\rangle$ which serves as the input of our GAT-PCM model (cf. Fig.~\ref{fig:arch}(d)). Specifically, for each edge in the microstructure, we insert a constraint-cost node $v_c\in V_C$ which corresponds to the constraint cost between the related pair of variable assignments. For each constraint function $f\in F$, we also create a function node $v_f \in V_F$ in the graph, and each related constraint-cost node will be directed to $v_f$. Note that $v_f$ can be regarded as the \emph{proxy} of all related constraint-cost nodes.
Besides, variable-assignment nodes related to $\Gamma$ will also be removed from the tripartite graph since they are \emph{subsumed} by their related constraint-cost nodes.

Finally, we note that the loopy nature of undirected microstructure may lead to missense propagation and potentially cause an oversmoothing problem \cite{li2019deepgcns}. For example, $\langle x_3, R\rangle$ should be independent of $\langle x_3, L\rangle$ since they are two different assignments of the same variable. However, $\langle x_3, L\rangle$ could indirectly influence $\langle x_3, R\rangle$ through multiple paths (e.g., $\langle x_3, L\rangle-\langle x_2, R\rangle-\langle x_3, R\rangle$) when applying GATs. Therefore, we require the tripartite graph to be directed and acyclic such that each constraint-cost node or variable-assignment node has a path to the target variable-assignment node. Specifically, we determine the directions between constraint-cost nodes and variable-assignment nodes through a two-phase procedure. First, we build a Directed Acyclic Graph (DAG) for a constraint graph induced by the set of unassigned variables such that every unassigned variable has a path to the target variable. To this end, we build a pseudo tree $PT$ \cite{freuder1985taking} with the target variable as its root and use $PT$ as the DAG where each node of $PT$ will be directed to its parent or pseudo parents. Second, for any pair of constrained variables $x_i$ and $x_j$ in the DAG, where $x_i$ is the precursor of $x_j$, and any related pair of variable assignments $\langle x_i,d_i\rangle$ and $\langle x_j,d_j\rangle$, we set the node of $\langle x_i,d_i\rangle$ to be the precursor of the constraint-cost node of $f_{ij}(d_i,d_j)$ and set the constraint-cost node of $f_{ij}(d_i,d_j)$ to be the precursor of the node of $\langle x_j,d_j\rangle$ in the tripartite graph.
Note that constraint-cost nodes related to a unary function will be set to be the precursor of their corresponding variable-assignment nodes.

For space complexity, given an instance with $|I|$ variables and maximum domain size of $d$, the directed acyclic tripartite graph has $O(d|I|)$ variable-assignment nodes, $O(|I|^2)$ function nodes and $O(d^2|I|^2)$ constraint-cost nodes.

\subsubsection{Graph Embeddings}
Given a directed tripartite graph representation, we use GATs to learn an embedding with supervised training (cf. Fig.~\ref{fig:arch}(e)). Each node $v_i$ has a four-dimensional initial feature vector $h^{(0)}_i\in H^{(0)}$ where the first three elements denote the one-hot encoding of node types (i.e., variable-assignment nodes, constraint-cost nodes, and function nodes) and the last element is set to be the constraint cost of $v_i$ if it is a constraint-cost node and otherwise, $0$. The initial feature matrix $H^{(0)}$ is then embedded through $T$ layers of the GAT. Formally, 
\begin{equation}
    H^{(t)} = \mathcal{M}_{\theta,(t)}(H^{(t-1)}),\quad t=1,\dots,T,\label{eq-embed}
\end{equation}
where $H^{(t)}$ is the embedding in the $t$-th timestep and $\mathcal{M}_{\theta,(t)}$ is the $t$-th layer of the GAT. Finally, given a target variable-assignment $x_m=d_m$ and a partial assignment $\Gamma$, we predict the optimal cost of the partially instantiated problem instance induced by $\Gamma\cup\{ x_m=d_m\}$ based on the embedding of the node of $v_m=\langle x_m, d_m\rangle\in V_X$ and the accumulated embedding of all function nodes of the tripartite graph as follows:
\begin{equation}
    \hat{c}_m=\mathcal{M}_{\theta, (T+1)}(h_m^{(T)}\oplus \sum_{v_i\in V_F}h_i^{(T)}), \label{eq-mlp}
\end{equation}
where $\mathcal{M}_{\theta, (T+1)}$ is a fully-connected layer and $\oplus$ is the concatenation operation.

Note that, by our construction of the tripartite graph, function nodes are the proxies of all constraint-cost nodes and all the other variable-assignment nodes have been directed to the target variable-assignment node. Therefore, we do not need to include the embeddings of constraint-cost nodes and variable-assignment nodes except the target variable-assignment node in Eq.~(\ref{eq-mlp}).

\paragraph{Pretraining}
\begin{algorithm}[t]
\small
\caption{Offline pretraining procedure}\label{algo:train_c}
\begin{algorithmic}[1]
\Require number of training epochs $N$, number of training iterations $K$, problem distribution $\mathcal{P}$, optimal DCOP algorithm $\mathcal{A}$, capacitated FIFO buffer $\mathcal{B}$
\For{$n=1,\dots, N$}
\Statex \textbf{Phase I: generating labelled data}
\State $P\equiv\langle I,X,D,F\rangle\sim\mathcal{P}$, $PT\gets$ build a pseudo tree for $P$
\ForAll{$x_i\in X$}
\State $Sep(x_i)\gets $ anc. connecting $x_i$ and its desc. in $PT$
\ForAll{context $\Gamma_i\in \Pi_{x_j\in Sep(x_i)}D_j$}
\ForAll{$d_i\in D_i$}
\State $P^\prime\gets\textsc{Reduce}(P,\Gamma_i,x_i=d_i)$
\State $c^*\gets \mathcal{A}(P^\prime)$, $\mathcal{B}\gets\mathcal{B}\cup\{\langle P,\Gamma_i,x_i=d_i,c^*\rangle\}$
\EndFor
\EndFor
\EndFor
\Statex \textbf{Phase II: training the model}
\For{$k=1,\dots,K$}
\State $B\gets\text{ sample a batch of data from }\mathcal{B}$
\State train the model $\mathcal{M}_\theta$ to minimize Eq.~(\ref{eq-loss})
\EndFor
\EndFor
\Return $\mathcal{M}_\theta$
\end{algorithmic}
\end{algorithm}
Algorithm 1 sketches the training procedure. For each epoch, we first generate labelled data (i.e., partial assignments, target assignments and corresponding optimal costs) in phase I and then train our model in phase II.

Specifically, we first sample a DCOP instance $P$ from the problem distribution $\mathcal{P}$. For each target variable $x_i$, instead of randomly generating partial assignments, we build a pseudo tree $PT$ and use its contexts w.r.t. $PT$ as partial assignments (line 2-5). In this way, we avoid redundant partial assignments by focusing only on the variables that are constrained with $x_i$ or its descendants. After obtaining the subproblem rooted at $x_i$ (cf. procedure \textsc{Reduce}), we apply any off-the-shelf optimal DCOP algorithm $\mathcal{A}$ to solve $P^\prime$ to get the optimal cost $c^*$ (line 6-8).

Each tuple of partial assignment, target assignment, optimal cost and problem instance will be stored in a capacitated FIFO buffer $\mathcal{B}$. 
In phase II, we uniformly sample a batch $B$ of data from the buffer to train our model using the mean squared error loss:
\begin{equation}
\resizebox{.9\hsize}{!}{
    $\mathcal{L}(\theta)=\frac{1}{|B|}\sum_{\langle P,\Gamma,x_i=d_i,c^*\rangle\in B}(\mathcal{M}_\theta(P,x_i=d_i;\Gamma)-c^*)^2.$
    }\label{eq-loss}
\end{equation}

\algrenewcommand\algorithmicprocedure{\textbf{When}}
\begin{algorithm}[!t]
\small
\caption{Distributed embedding schema for agent $i$}\label{algo:train_d}
\begin{algorithmic}[1]
\Require trained model $\mathcal{M}_\theta$, precursors $P_i$, successors $S_i$, target assignment $x_m=d_m$, initial variable-assignment node feature $h_X^{(0)}$, initial function node feature $h_F^{(0)}$, one-hot encoding for constraint-cost node $h_C^{(0)}$
\Procedure{Initialization}{}:
\State $Cache_i\gets []$, $H_i^{(0)}\gets $ empty tensor
\For{$t=1,\dots, T$}
$Cache_i[t]\gets $ empty map
\EndFor
\If{$x_i=x_m$}
$H_i^{(0)}\gets \textsc{Stack}(H_i^{(0)},h_X^{(0)})$
\Else
\ForAll{$d_i\in D_i$}
$H_i^{(0)}\gets \textsc{Stack}(H_i^{(0)},h_X^{(0)})$
\EndFor
\EndIf
\ForAll{$j\in S_i$}
\State $H_i^{(0)}\gets \textsc{Stack}(H_i^{(0)},h_F^{(0)})$
\ForAll{cost value $c_{ij}\in f_{ij}(\cdot,\cdot)$}
\State $H_i^{(0)}\gets \textsc{Stack}(H_i^{(0)},h_C^{(0)}\oplus c_{ij})$
\EndFor
\EndFor
\ForAll{$j\in P_i$}
\ForAll{cost value $c_{ji}\in f_{ji}(\cdot,\cdot)$}
\State $H_i^{(0)}\gets \textsc{Stack}(H_i^{(0)},h_C^{(0)}\oplus c_{ji})$
\EndFor
\EndFor
\State $\ell_i\gets $ zero vector, $t_i\gets 1$, $H_i^{(t_i)}\gets \mathcal{M}_{\theta, (t_i)}(H_i^{(t_i-1)})$
\State send $H^{(t_i)}_{i}[f_{ij}(\cdot,\cdot)]$ to $j,\; \forall j\in S_i$
\If{$P_i=\emptyset$}
\For{$t_i=2,\dots, T$}
\State $H^{(t_i)}\gets \mathcal{M}_{\theta, (t_i)}(H^{(t_i-1)})$
\If{$t_i<T$}
\State send $H^{(t_i)}_{i}[f_{ij}(\cdot,\cdot)]$ to $j,\; \forall j\in S_i$
\EndIf
\EndFor
\State $\ell_i\gets\sum_{j\in S_i}H^{(T)}_i[f_{ij}]$, send $\ell_i$ to $j^\prime\in S_i$
\EndIf
\EndProcedure
\Procedure{Receive}{} embedding $\bar{H}^{(t_j)}$ from $j\in P_i$:
\State $Cache_i[t_j][j]\gets \bar{H}^{(t_j)}$
\If{$|Cache_i[t_i]|=|P_i|$}

\ForAll{${j^\prime}\in P_i$} 
\State $H_{i}^{(t_i)}[f_{ij^\prime}(\cdot,\cdot)]\gets Cache[t_i][{j^\prime}]$ 
\EndFor
\State $t_i\gets t_i+1$
\State $H^{(t_i)}\gets \mathcal{M}_{\theta, (t_i)}(H^{(t_i-1)})$
\If{$t_i<T$}
\State send $H^{(t_i)}_{i}[f_{ij^\prime}(\cdot,\cdot)]$ to $j^\prime,\; \forall j^\prime\in S_i$
\Else
\State $\ell_i\gets\sum_{j^{\prime\prime}\in S_i}H^{(T)}_i[f_{ij^{\prime\prime}}]$, send $\ell_i$ to $j^\prime\in S_i$
\EndIf
\EndIf
\EndProcedure
\Procedure{Receive}{} accum. embedding $\ell_j$ from $j\in P_i$:
\If{$x_i\ne x_m$} 
\State send $\ell_j$ to ${j^\prime}\in S_i$
\Else 
\State $\ell_i\gets \textsc{Add}(\ell_i,\ell_j)$
\If{all accum. embeddings have arrived}
\State computes Eq.~(\ref{eq-dis-mlp})
\EndIf
\EndIf
\EndProcedure
\end{algorithmic}
\end{algorithm}

\paragraph{Distributed Embedding Schema}
Different from pretraining stage where the model has access to all the knowledge (e.g., variables, domains, constraints, etc.) about the instance to be solved, an agent in real-world scenarios usually can only be aware of its local problem due to privacy concern and/or geographical limitation, posing a significant challenge when applying our model to solve DCOPs. Also, centralized model inference could overwhelm a single agent. Therefore, we aim to develop a distributed schema for model inference in which each agent only uses its local knowledge to cooperatively compute Eq.~(\ref{eq-embed}) and Eq.~(\ref{eq-mlp}).


We
exploit the directed and acyclic nature of our tripartite graph and propose an efficient Distributed Embedding Schema (DES) in Algorithm~2. The general idea is that each agent maintains the embeddings w.r.t. its local problem. Specifically, an agent $i$ maintains the following components: (1) its own variable-assignment nodes and (induced) unary constraint-cost and function nodes; (2) all function nodes $f_{ij}$ where $x_j$ is a successor of $x_i$; and (3) all constraint-cost nodes $f_{ij}(d_i,d_j)$ where $x_j$ is a successor of $x_i$. Each time the agent updates the local embeddings via a single step of model inference after receiving the embeddings from its precursors. Taking the tripartite graph in Fig.~\ref{fig:arch}(d) as an example, $x_2$ maintains embeddings for $\langle x_2,R\rangle$, $f_{12}(R,R)$, $f_{12}$, $f_{23}(R,R)$ and $f_{23}(R,L)$. To update its local embeddings for $\langle x_2,R\rangle$, $f_{12}(R,R)$ and $f_{12}$, $x_2$ only needs one step of model inference after receiving the latest embedding of constraint-cost nodes $f_{23}(R,R)$ and $f_{23}(R,L)$ from its precursor $x_3$. 


Next, we give details about the schema. First, we use primitive \textsc{Stack} to concatenate the initial features of local nodes\footnote{We omit unary functions for simplicity.} to construct the initial embeddings $H^{(0)}_i$ (line 3-13). After that, agent $i$ computes its first round embeddings and sends the updated embeddings of constraint-cost nodes to each of its successors (line 14-15). If agent $i$ is a source node, i.e., $P_i=\emptyset$, it directly updates the subsequent embeddings and sends the constraint-cost node embeddings at each timestep to its successors (line 16-20) since it does not need to wait for the embeddings from its precursors. Besides, the agent also sends the local accumulated function node embedding $\ell_i$ to one of its successors (line 21).

After receiving the constraint-cost node embeddings from its precursor $j$, agent $i$ temporarily stores the embeddings to $Cache_i$ according to the timestamp $t_j$ (line 23). If all the precursors' constraint-cost node embeddings for the $t_i$-th layer have arrived, agent $i$ updates the local embedding $H_i^{(t_i)}$ with those embeddings stored in $Cache_i[t_i]$ (line 24-26). Then the agent computes the embeddings $H_i^{(t_i+1)}$ and sends the up-to-date embeddings to its successors (line 27-30). If all GAT layers are exhausted, the agent computes the local accumulated function node embedding $\ell_i$ and sends it to one of its successors (line 31-32). After receiving an accumulated function-node embedding, agent $i$ either directly forwards the embedding to one of its successors or adds to its own accumulated function-node embedding, depending on whether it is the target agent (line 34-37). After received the accumulated embedding messages of all the other agents, the target agent $m$ outputs the predicted optimal cost by
\begin{equation}
    \hat{c}_m=\mathcal{M}_{\theta, (T+1)}(H_m^{(T)}[\langle x_m,d_m\rangle]\oplus \ell_m), \label{eq-dis-mlp}
\end{equation}
where $H_m^{(T)}[\langle x_m,d_m\rangle]$ is the embedding for variable-assignment node $\langle x_m,d_m\rangle$ in $H_m^{(T)}$ (line 38-39).

\newtheorem{proposition}{Proposition}
\newtheorem{lemma}{Lemma}
We now show the soundness and complexity of DES. We first show that DES results in the same embeddings as its centralized counterpart.
\begin{lemma}
In DES, each agent $i$ with $P_i\ne \emptyset$ receives exactly $T-1$ constraint-cost node embedding messages from $j,\forall j\in P_i$, one for each timestep $t_j=1,\dots,T-1$. 
\end{lemma}
\begin{proof}
Consider the base case where all the precursors are source i.e., $P_j=\emptyset,\forall j\in P_i$. Since it cannot receive a embedding from other agent, each precursor $j$ sends exactly $T-1$ constraint-cost node embeddings to $i$, one for each timestep $t_j=1,\dots,T-1$ according to line 15, 19-20.

Assume that the lemma holds for all $j\in P_i$ with $P_j\ne\emptyset$. By assumption, the condition of line 24 holds for $t_j=1,\dots,T-1$ and hence precursor $j$ sends embedding to $i$ for $t_j=2,\dots,T-1$ (line 27-30). Together with the embedding sent in line 15, each precursor $j$ sends $T-1$ constraint-cost node embedding messages to $i$ in total, one for each timestep $t_j=1,\dots,T-1$, which concludes the lemma. 
\end{proof}

\begin{lemma}
For any agent $i$ and timestep $t=1,\dots,T$, after performing DES, its local embeddings are the same as the ones in $H^{(t)}$. I.e., $H^{(t)}_i[\langle x_i,d_i\rangle]=H^{(t)}[\langle x_i,d_i\rangle]$, $H_i^{(t)}[f_{ij}(d_i,d_j)]=H^{(t)}[f_{ij}(d_i,d_j)]$, $H_i^{(t)}[f_{ij}]=H^{(t)}[f_{ij}]$, $\forall d_i\in D_i, x_j\in S_i, d_j\in D_j$.
\end{lemma}
\begin{proof}
We only show the proof for variable-assignment nodes. Similar argument can be applied to constraint-cost nodes and function nodes.

In the first timestep, i.e., $t=1$, for each node $\langle x_i,d_i\rangle$, Eq.~(\ref{eq-embed}) computes $H^{(1)}[\langle x_i,d_i\rangle]$ based on the initial feature $H^{(0)}[f_{ij}(d_i,d_j)]$, $\forall j\in P_i, d_j\in D_j$, which is the same as in DES, i.e., line 11-14.
 
Assume that the lemma holds for $t>1$. Before computing the embeddings for $(t+1)$-th timestep, agent $i$ must have received the embedding $H_j^{(t)}[f_{ij}(d_i,d_j)]$, which equals to $H^{(t)}[f_{ij}(d_i,d_j)]$ according to the assumption, from $j, \forall j\in P_i, d_i\in D_i,d_j\in D_j$ (line 20, 30 23-26, Lemma 1). Therefore, agent $i$ computes $H_i^{(t+1)}[\langle x_i,d_i\rangle]$ according to $H^{(t)}[f_{ij}(d_i,d_j)],\forall j\in P_i, d_j\in D_j$, which is equivalent to Eq.~(\ref{eq-embed})). Consequently, $H_i^{(t+1)}[\langle x_i,d_i\rangle]=H^{(t+1)}[\langle x_i,d_i\rangle]$ and the lemma holds by induction.
\end{proof}
\begin{lemma}
For target agent $m$, after performing DES, $\ell_m=\sum_{v_i\in V_F}h_i^{(T)}$.
\end{lemma}
\begin{proof}
We prove the lemma by showing each agent sends exactly one accumulated embedding message w.r.t. its local function nodes to one of its successors (i.e., line 21 and 32). It is trivial for the agents without precursor since they do not receive any message (line 28) and only send one accumulated embedding message by the end of procedure \textsc{Initialization} (line 21).

Consider an agent $i$ with $P_i\ne\emptyset$. According to Lemma 1, $i$ executes line 27-32 for $T-1$ times. Given the initial value of 1 (line 14), $t_i$ will eventually equal to $T$, implying line 32 will be executed only once. Since it does not perform line 21, $i$ sends exactly one accumulated embedding message w.r.t. its local function nodes.

Since by construction each agent in the DAG has a path to the target agent $m$, all the accumulated embeddings will be forwarded to $m$ (line 34-37). Therefore, by Lemma 2,
\[
\ell_m=\sum_{i\in I}\sum_{j\in S_i}H_i^{(T)}[f_{ij}]=\sum_{i\in I}\sum_{j\in S_i}H^{(T)}[f_{ij}].
\]
Note that $\forall f_{ij}\in F$, it must be either the case $j\in S_i$ if $i\prec j$ or the case $i\in S_j$ if $j\prec i$ in the DAG. Hence,
\[
\ell_m=\sum_{i\in I}\sum_{j\in S_i}H^{(T)}[f_{ij}]=\sum_{f_{ij}\in F}H^{(T)}[f_{ij}]=\sum_{v_i\in V_F}h_i^{(T)}.
\]
\end{proof}

Then we show the soundness of our DES as follows:
\begin{proposition}
DES is sound, i.e., Eq.~(\ref{eq-dis-mlp}) returns the same result as Eq.~(\ref{eq-mlp}).
\end{proposition}
\begin{proof}
According to the Lemma 2 and Lemma 3, by the end of DES, the target agent has the same variable-assignment embedding and accumulated function node embedding as the ones computed by Eq.~(\ref{eq-embed}). Therefore, Eq.~(\ref{eq-dis-mlp}) is equivalent to Eq.~(\ref{eq-mlp}).
\end{proof}

Finally, we show the complexity of our DES as follows:
\begin{proposition}
Each agent in DES requires $T$ steps of model inference, $O(|I|d^2)$ spaces, and communicates $O(T|I|d^2)$ information.
\end{proposition}
\begin{proof}
By line 14, 27-28 and Lemma 1, each agent performs $T$ times of model inference. Each agent $i$ needs to maintain embedding for $O(d)$ assignment-variable nodes (line 4-6), $O(|S_i|d^2)+|P_i|d^2)$ constraint-cost nodes (line 7, 9-13), and $O(|S_i|)$ function nodes (line 8). Since in the worst case, the agent is constrained with all the other $|I|-1$ agents, $i$'s space complexity is $O(|I|d^2)$. Finally, since for each timestep $t_i=1,\dots,T-1$ agent $i$ sends the constraint-cost node embeddings to its successors, its communication overhead is $O(T|S_i|d^2)=O(T|I|d^2)$.
\end{proof}
\begin{figure*}
    \centering
    \begin{minipage}{.33\linewidth}
      \centering
      \includegraphics[width=\linewidth]{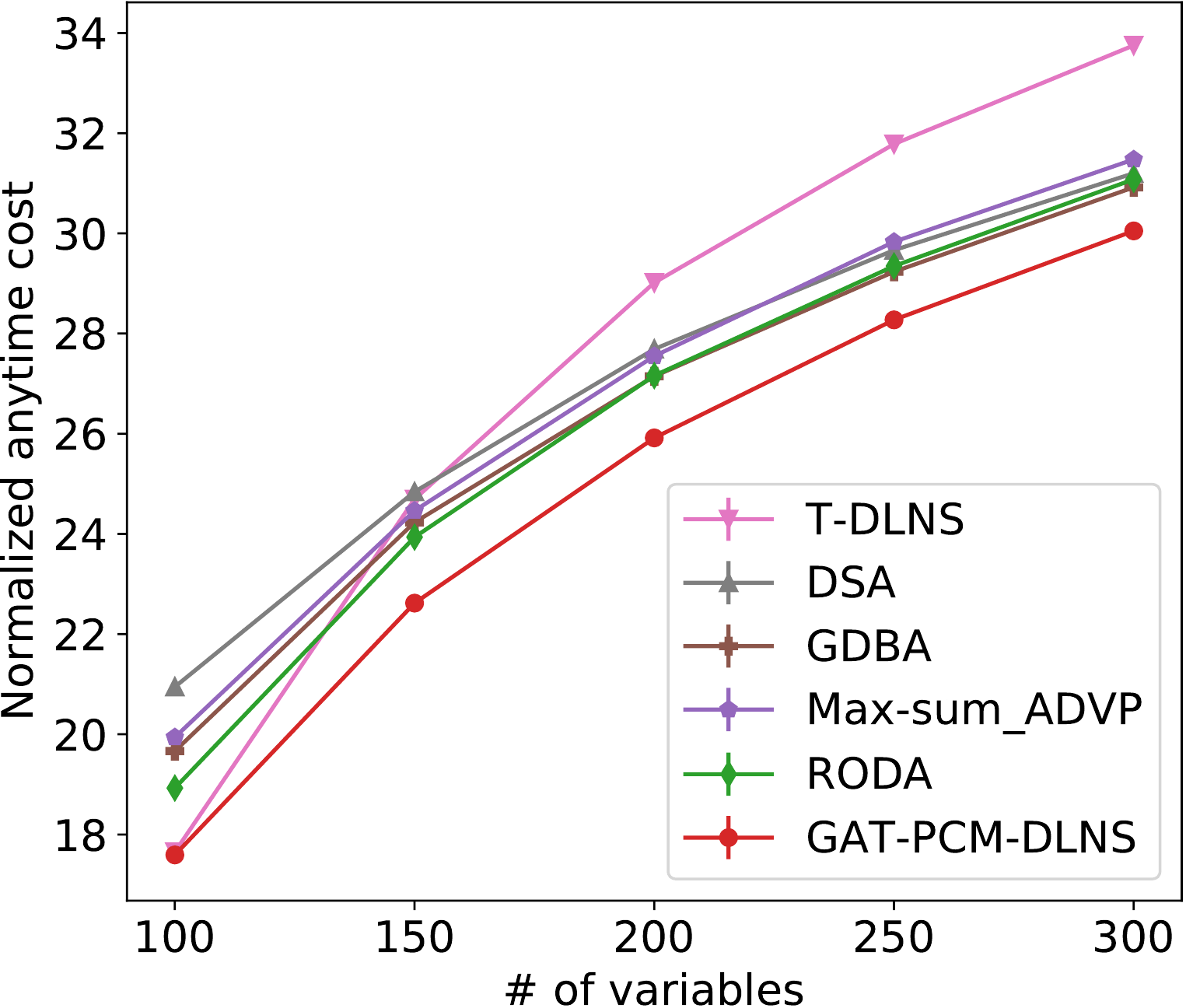}\\
      (a) $p_1=0.05$, $d=10$
    \end{minipage}
    \begin{minipage}{.33\linewidth}
      \centering
      \includegraphics[width=\linewidth]{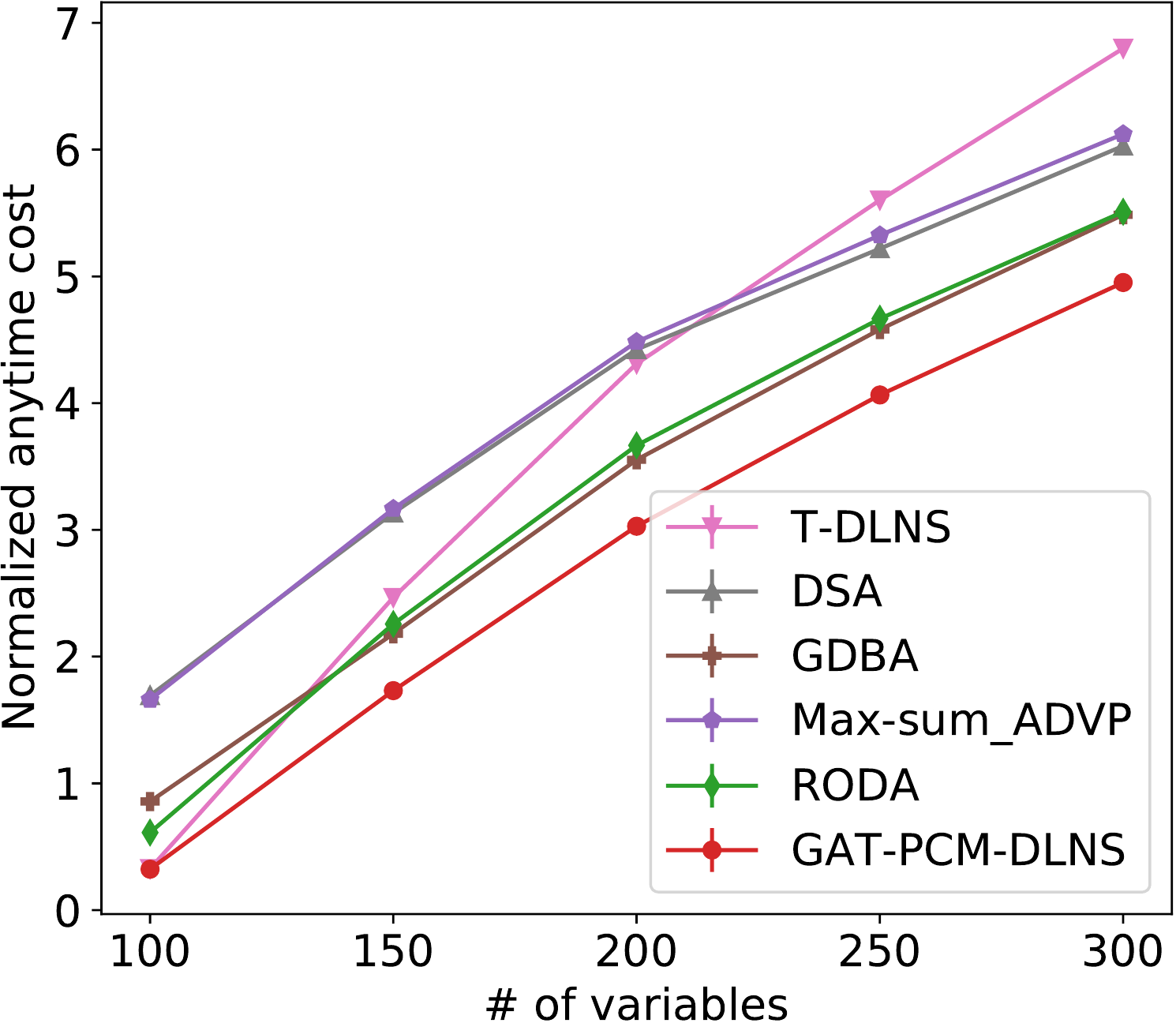}\\
      (b) $p_1=0.05$, $d=3$
    \end{minipage}
    \begin{minipage}{.33\linewidth}
      \centering
      \includegraphics[width=\linewidth]{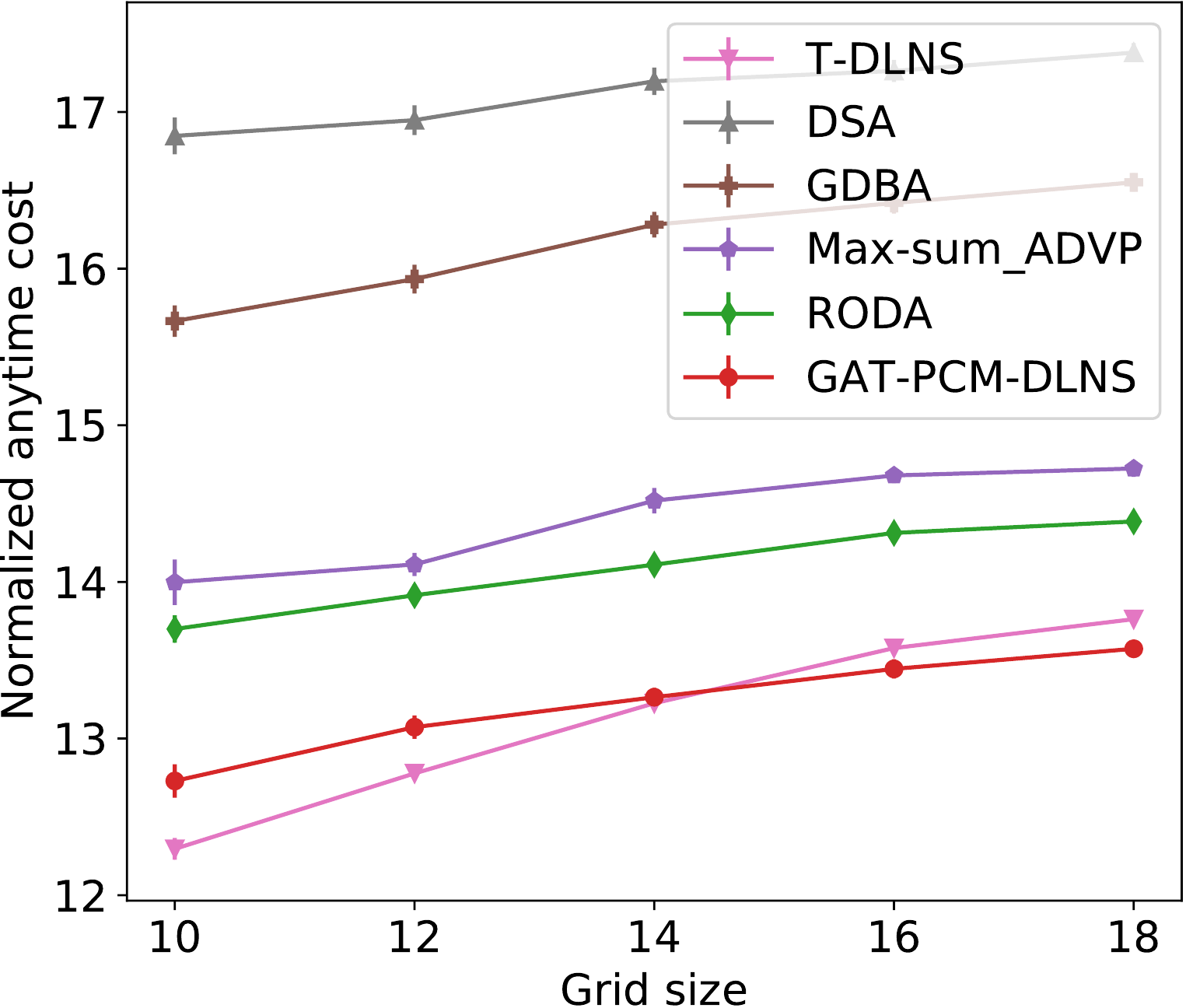}\\
      (c) $d=10$
    \end{minipage}
    \caption{Solution quality comparisons: (a) random DCOPs; (b) weighted graph coloring problems; (c) grid networks}\label{exp-dcop}
\end{figure*}

\paragraph{GAT-PCM as Heuristics}
Since our model GAT-PCM predicts the optimal cost of a target assignment given a partial assignment, it can serve as a general heuristic to boost the performance of a wide range of DCOP algorithms where the core operation is to evaluate the quality of an assignment. We consider two kinds of well-known algorithms and show how our model can boost them as follows:
\begin{itemize}
    \item Local search. A key task in local search is to find good assignments for a set of variables given the other variables' assignments. For example, in Distributed Large Neighborhood Search (DLNS) \cite{hoang2018large} , each round a subroutine is called to solve a subproblem induced by the destroyed variables (also called repair phase). Currently, DPOP \cite{petcuscalable} is used to solve a tree-structured relaxation of the subproblem, which ignores a large proportion of constraints and thus leads to poor performance on general problems. 
    Instead, we use our GAT-PCM to solve the subproblem without relaxation (i.e., all constraints between all pairs of destroyed variables are included) since the overhead is polynomial in the number of agents (cf. Proposition 2). Specifically, for each connected subproblem, we assume a variable ordering (e.g., lexicographical ordering, pseudo tree). Then we greedily assign each variable according to the costs predicted by GAT-PCM, i.e., we select an assignment with the smallest predicted cost for each variable.
    
    \item Backtracking search. Domain ordering is another important task in backtracking search for DCOPs. Previously, domain ordering utilizes local information only, e.g., prioritizing the assignment with minimum conflicts w.r.t. each unassigned variable \cite{frost1995look} or querying a lower bound lookup table. On the other hand, our GAT-PCM offers a more general and systematic way for domain ordering. Specifically, for an unassigned variable, we could query GAT-PCM for the optimal cost of each assignment under the current partial assignment and give the priority to the one with minimum predicted cost.
\end{itemize}

\section{Empirical Evaluation}
In this section, we perform extensive empirical studies. We begin with introducing the details of experiments and pretraining stage. Then we analyze the results and demonstrate the capability of our GAT-PCM to boost DCOP algorithms.
\paragraph{Benchmarks} We consider four types of benchmarks in our experiments, i.e., random DCOPs, scale-free networks, grid networks, and weighted graph coloring problems. For random DCOPs and weighted graph coloring problems, given density of $p_1\in (0,1)$, we randomly create a constraint for a pair of variables with probability $p_1$. For scale-free networks, we use the BA model \cite{barabasi1999emergence} with parameter $m_0$ and $m_1$ to generate constraint relations: starting from a connected graph with $m_0$ vertices, a new vertex is connected to $m_1$ vertices with a probability which is proportional to the degree of each existing vertex in each iteration. Besides, variables in a grid network are arranged into a 2D grid, where each variable is constrained with four neighboring variables excepts the ones located at the boundary. Finally, for each constraint in random DCOPs, scale-free networks and grid networks, we uniformly sample a cost from $[0,100]$ for each pair of variable-assignments. Differently, constraints of the weighted graph coloring problems incur a cost which is also uniformly sampled from $[0,100]$ if two constrained variables have the same assignment.

\paragraph{Baselines}We consider four types of baselines: local search, belief propagation, region optimal method, and large neighborhood search. We use DSA \cite{Zhang2005Distributed} with $p=0.8$ and GDBA \cite{okamoto2016distributed} with $\langle M,NM, T\rangle$ as two representative local search methods, Max-sum$\_$ADVP \cite{zivan2017balancing} as a representative belief propagation method, RODA \cite{grinshpoun2019privacy} with $t=2,k=3$ as a representative region optimal method, and T-DLNS \cite{hoang2018large} with destroy probability $p=0.5$ as a representative large neighborhood search method.

All experiments are conducted on an Intel i9-9820X workstation with GeForce RTX 3090 GPUs. For each data point, we average the results over 50 instances and report standard error of the mean (SEM) as confidence intervals.
\begin{figure}
    \centering
    \includegraphics[width=.6\linewidth]{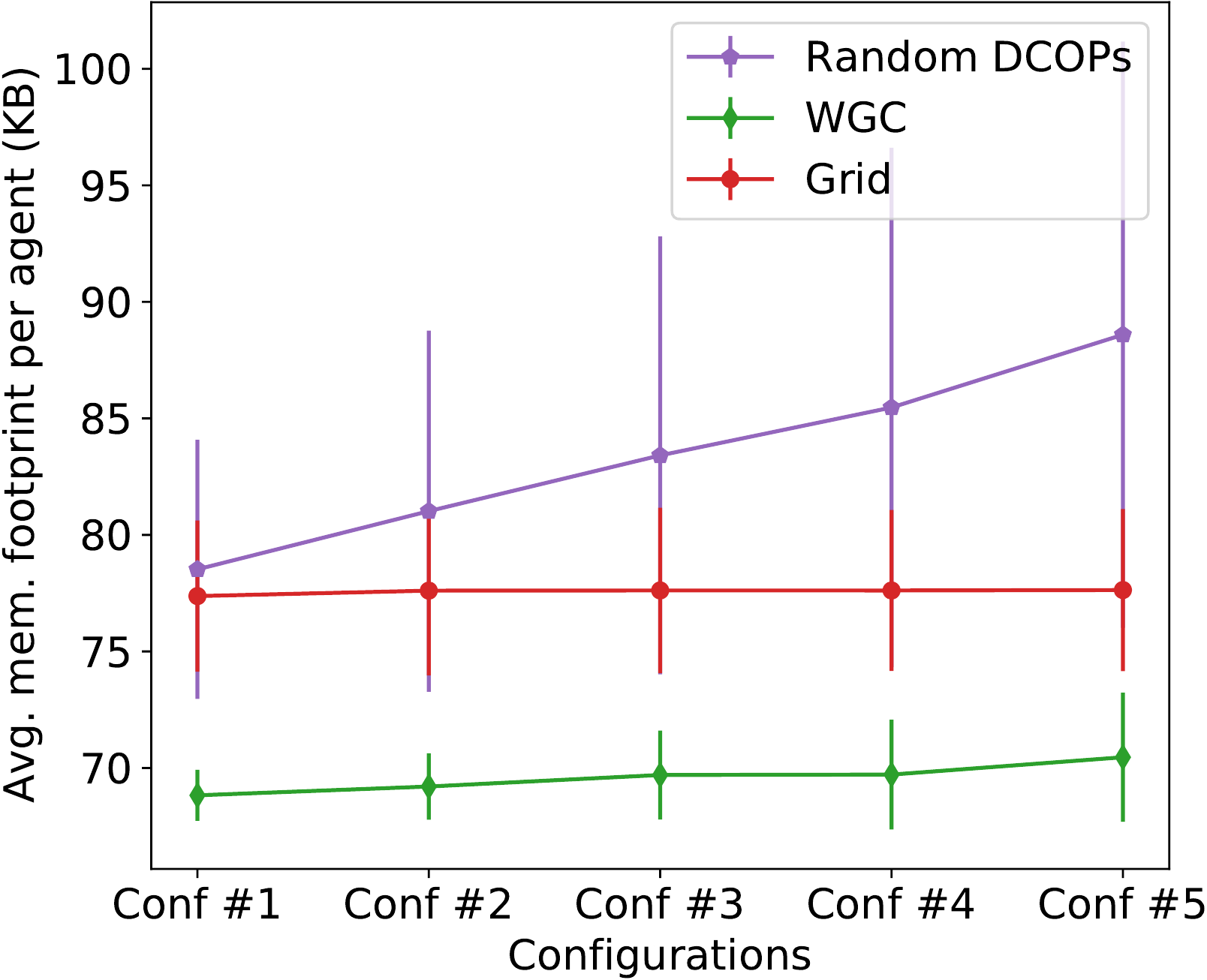}
    \caption{Memory footprint of GAT-PCM-DLNS}\label{fig:mem}
\end{figure}

\paragraph{Implementation and Hyperparameters} Our GAT-PCM model has four GAT layers (i.e., $T=4$). Each layer in the first three layers has 8 output channels and 8 heads of attention, while the last layer has 16 output channels and 4 heads of attention. Each GAT layer uses \texttt{ELU} \cite{clevertUH15} as the activation function. In the pretraining stage, we consider a random DCOP distribution with $|I|\in[15,30]$, $d\in[3,15]$ and $p_1\in[0.1,0.4]$. Finally, we use DPOP \cite{petcuscalable} to generate the optimal cost labels. 

For hyperparameters, we set the batch size and the number of training epochs to be 64 and 5000, respectively. Our model was implemented with the PyTorch Geometric framework \cite{lenssen2019} and the model was trained with the Adam optimizer \cite{kingma2014adam} using the learning rate of 0.0001 and a $5\times 10^{-5}$ weight decay ratio.

\paragraph{Results} In the first set of experiments, we evaluate the performance of our GAT-PCM when combined with the DLNS framework, which we name it GAT-PCM-DLNS, in solving large-scale DCOPs. We run GAT-PCM-DLNS with destroy probability of 0.2 for 1000 iterations and report the normalized anytime cost (i.e., the best solution cost divided by the number of constraints) as the result. Fig.~\ref{exp-dcop} presents the results of solution quality where all baselines run for the same simulated runtime as GAT-PCM-DLNS. It can be seen that DSA explores low-quality solutions since it iteratively approaches a Nash equilibrium, resulting in 1-opt solutions similar to Max-sum$\_$ADVP. GDBA improves by increasing the weights when agents get trapped in quasi-local minima. RODA finds solutions better than 1-opt by coordinating the variables in a coalition  of size 3. T-DLNS, on the other hand, tries to optimize by optimally solving a tree-structured relaxation of the subproblem induced by the destroyed variables in each round. However, T-DLNS could ignore a large proportion of constraints and therefore perform poorly when solving complex problems (e.g., the problems with more than 200 variables). Differently, our GAT-PCM-DLNS solves the induced subproblem without relaxation, leading to a significant improvement over the state-of-the-arts when solving unstructured problems (i.e., Fig.~\ref{exp-dcop}(a-b)). Interestingly, T-DLNS achieves the best performance when solving small grid networks. That is because the variables in the problem are under-constrained and T-DLNS only needs to drop few edges to obtain a tree-structured problem. In fact, the average degree in these problems is less than 3.8. However, our GAT-PCM-DLNS still outperforms T-DLNS when the grid size is higher than 14.
\begin{figure}
    \centering
    \begin{minipage}{.495\linewidth}
      \centering
      \includegraphics[width=\linewidth]{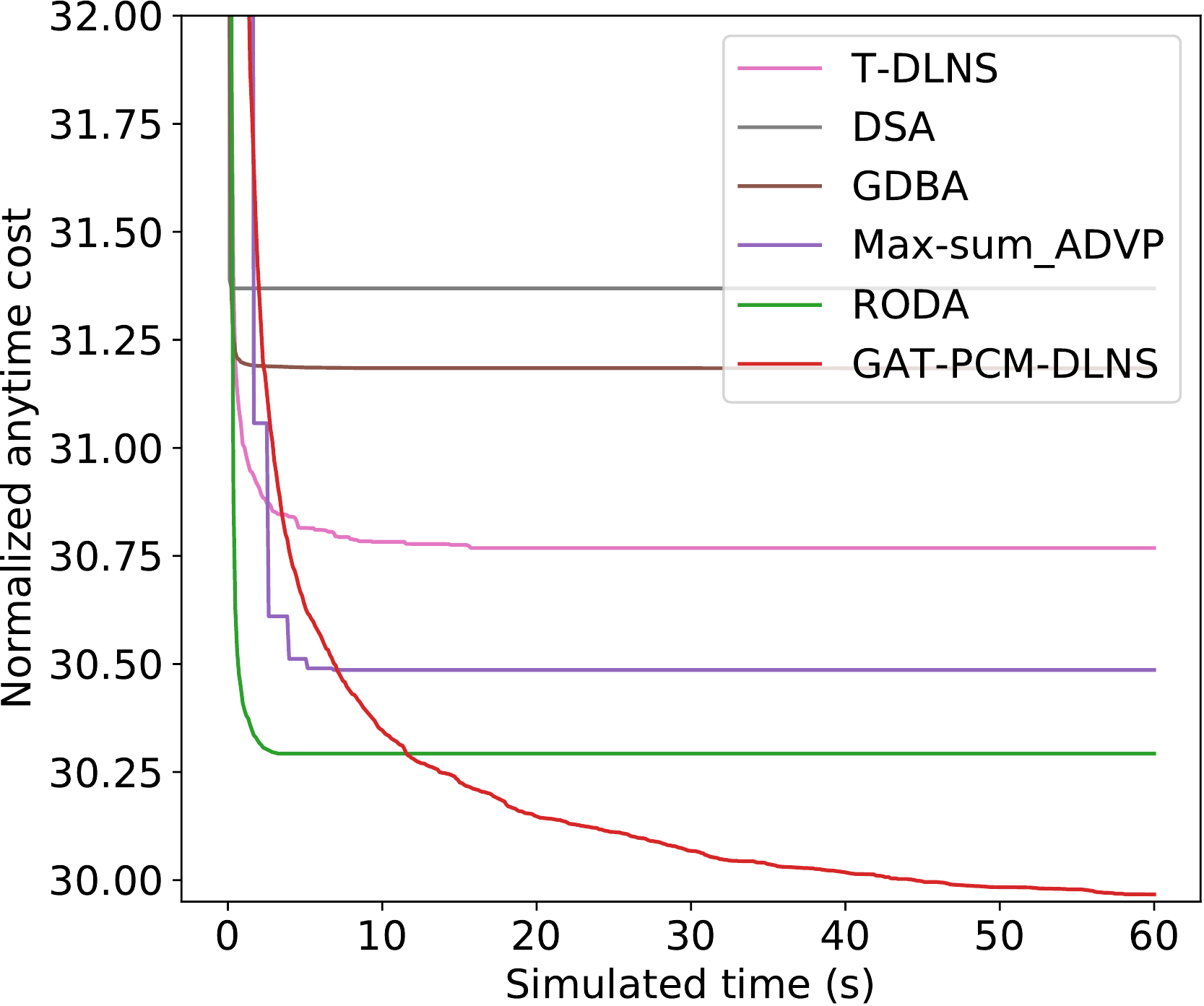}\\
      (a) random DCOPs
    \end{minipage}
    \begin{minipage}{.495\linewidth}
    \centering
    \includegraphics[width=\linewidth]{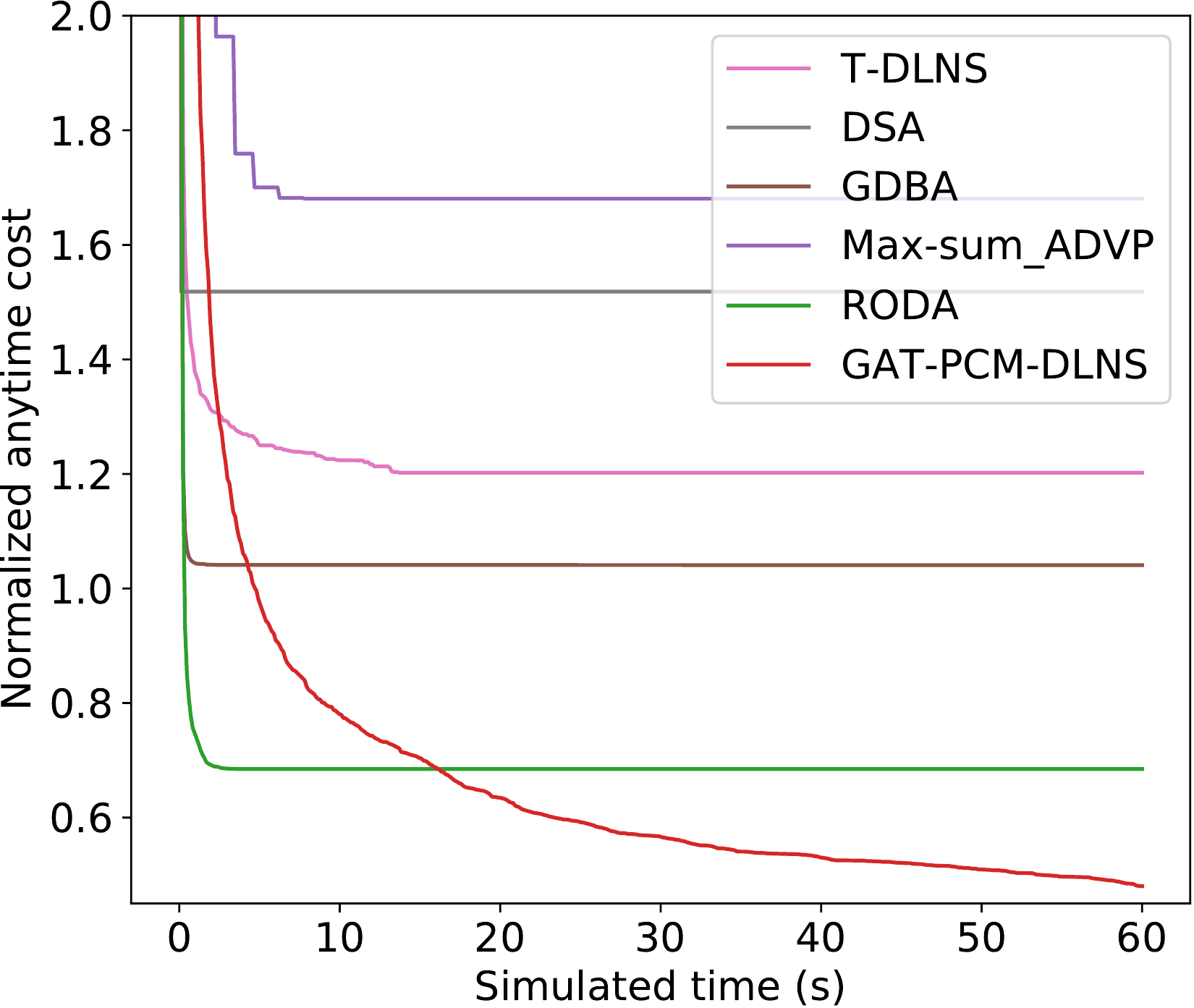}\\
      (b) weighted graph coloring
    \end{minipage}
    \caption{Convergence analysis}\label{fig:convergence}
\end{figure}

We display the average memory footprint per agent of GAT-PCM-DLNS in the first set of experiments in Fig \ref{fig:mem}, where ``Conf \#1'' to ``Conf \#5'' refer the growing complexity of each experiment. Specifically, the memory overhead of each agent consists of two parts, i.e., storing the pretrained model and local embeddings. The former consumes about 60KB memory, while the latter requires space proportional to the number of agents and the size of each constraint matrix (cf. Prop. 2). It can be concluded that our method has a modest memory requirement and scales up to large instances well in various settings. In particular, our method has a (nearly) constant memory footprint when solving grid network problems since each agent is constrained with at most four other agents regardless of the grid size.

To investigate how fast our GAT-PCM-DLNS finds a good solution, we conduct a convergence analysis which measures the performance in terms of simulated time \cite{sultanik2008dcopolis} on the problems with $|I|=1000,p_1=0.005$ and $d=3$ and present the results in Fig.~\ref{fig:convergence}. It can be seen that local search algorithms including DSA and GDBA quickly converge to a poor local optimum, while RODA finds a better solution in the first three seconds. T-DLNS slowly improves the solution but is strictly dominated by RODA. In contrast, our GAT-PCM-DLNS improves much steadily, outperforming all baselines after 18 seconds. 
\begin{figure}
    \centering
    \begin{minipage}{.495\linewidth}
    \centering
    \includegraphics[width=\linewidth]{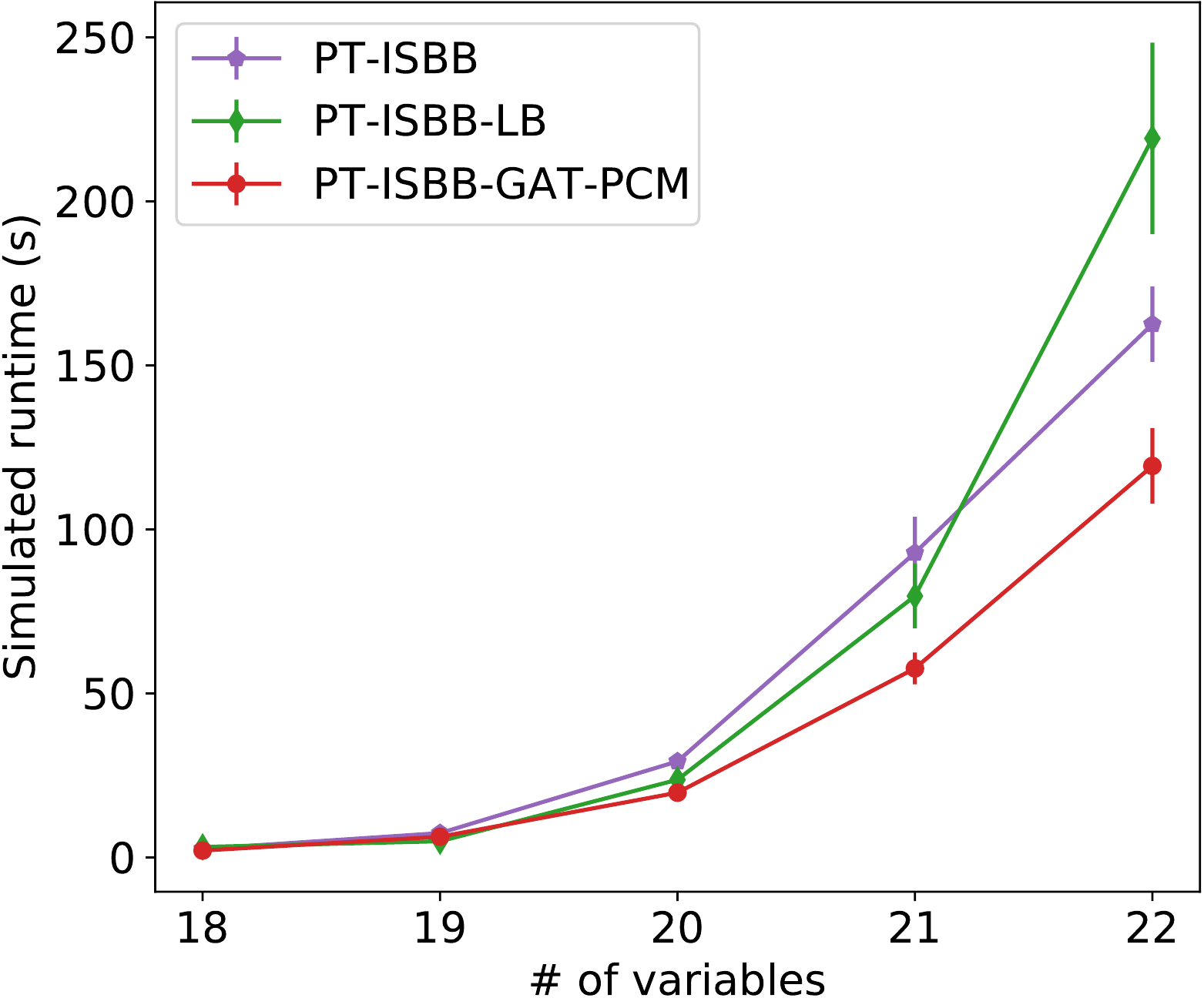}\\
      (a) random DCOPs
    \end{minipage}
    \begin{minipage}{.495\linewidth}
    \centering
    \includegraphics[width=\linewidth]{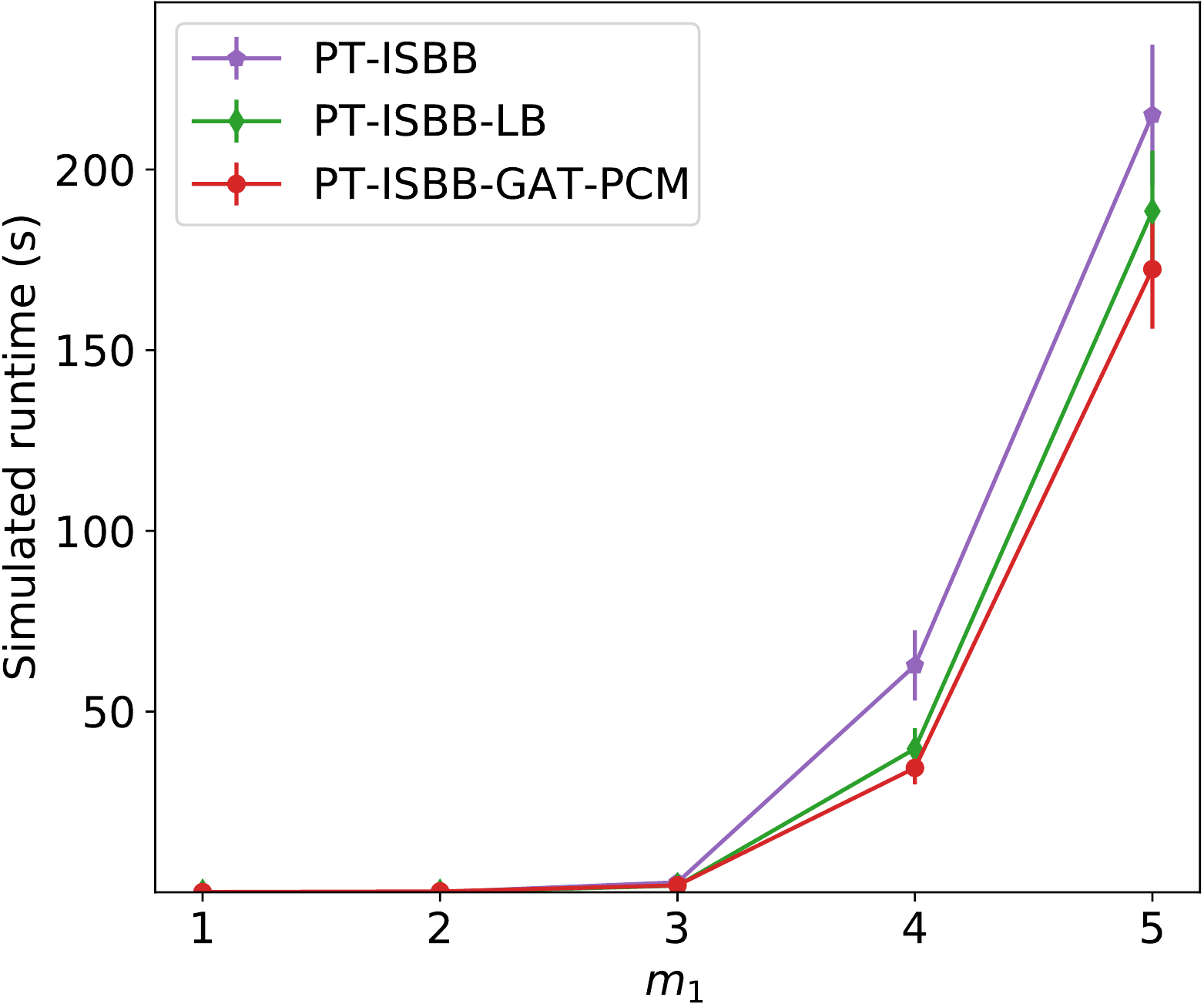}\\
      (b) scale-free networks
    \end{minipage}
    \caption{Runtime comparison on the problems with $d=5$}\label{fig:search}
\end{figure}

Finally, we demonstrate the merit of our GAT-PCM in accelerating backtracking search for random DCOPs with $p_1=0.25$ and scale-free networks with $|I|=18, m_0=5$ by conducting a case study on the symmetric version of PT-ISABB \cite{dengCCJL19} (referred as PT-ISBB) and present the results in Fig.~\ref{fig:search}. Specifically, we set the memory budget $k=2$ and compare the simulated runtime of PT-ISBB using three domain ordering generation techniques: alphabetically, lower bound lookup tables (PT-ISBB-LB), and our GAT-PCM (PT-ISBB-GAT-PCM). For PT-ISBB-GAT-PCM, we only perform domain ordering for the variables in the first three levels in a pseudo tree. It can be observed that the backtracking search with alphabetic domain ordering performs poorly and is dominated by the one with the lower bound induced domain ordering in the most cases. Notably, when solving the problems with 22 variables, PT-ISBB-LB exhibits the worst performance, because the lower bounds generated by approximated inference are not tight in complex problems, and hence the induced domain ordering may not prioritize promising assignments properly. On the other hand, our GAT-PCM powered backtracking search uses the predicted total cost of a subproblem as the criterion, resulting in a more efficient domain ordering and thus achieving the best results in solving complex problems.

\section{Conclusion}
In this paper, we present GAT-PCM, the first effective and general purpose deep pretrained model for DCOPs. We propose a novel directed acyclic graph representation schema for DCOPs and leverage the Graph Attention Networks (GATs) to embed our graph representations. Instead of generating heuristics for a particular algorithm, we train the model with optimally labelled data to predict the optimal cost of a target assignment given a partial assignment, such that GAT-PCM can be applied to boost the performance of a wide range of DCOP algorithms where evaluating the quality of an assignment is critical.
To enable efficient graph embedding in a distributed environment, we propose DES to perform decentralized model inference without disclosing local constraints, where each agent exchanges only the embedded vectors via localized  communication.
Finally, we develop several heuristics based on GAT-PCM to improve local search and backtracking search algorithms. Extensive empirical evaluations confirm the superiority of GAT-PCM based algorithms over the state-of-the-arts.

In future, we plan to extend GAT-PCM to deal with the problems with higher-arity constraints and hard constraints. Besides, since agents in DES exchange the embedded vectors instead of constraint costs, it is promising to extend our methods to an asymmetric setting \cite{grinshpoun2013asymmetric}.
\section{Acknowledgement}
This research was supported by the National Research Foundation, Singapore under its AI Singapore Programme (AISG Award No: AISG-RP-2019-0013), National Satellite of Excellence in Trustworthy Software Systems (Award No: NSOE-TSS2019-01), and NTU.
\bibliography{aaai22.bib}

\end{document}